\documentclass{article}

\usepackage[main, final]{neurips_2026}
\usepackage{thmtools}
\usepackage{wrapfig}

\usepackage[utf8]{inputenc} 
\usepackage[T1]{fontenc}    
\usepackage{hyperref}       
\usepackage{url}            
\usepackage{booktabs}       
\usepackage{amsfonts}       
\usepackage{nicefrac}       
\usepackage{microtype}      
\usepackage{xcolor}         

\usepackage{microtype}
\usepackage{graphicx}
\usepackage{subcaption}
\usepackage{booktabs} 

\usepackage{amsmath}
\usepackage{amssymb}
\usepackage{mathtools}
\usepackage{amsthm}
\usepackage{multirow}

\usepackage[capitalize,noabbrev]{cleveref}

\theoremstyle{plain}
\newtheorem{theorem}{Theorem}[section]

\theoremstyle{definition}
\newtheorem{definition}[theorem]{Definition}

\theoremstyle{remark}

\newcommand{\indep}{\perp \!\!\! \perp}
\newcommand{\dep}{\not\!\perp\!\!\!\perp}

\usepackage[textsize=tiny]{todonotes}

\title{On the Recoverability of Causal Relations from Bulk Gene Expression Data}

\author{Gongxu Luo$^1$,~~Boyang Sun$^{1}$,~~Kun Zhang$^{1,2}$\\
$^1$ Mohamed bin Zayed University of Artificial Intelligence, 
$^2$ Carnegie Mellon University\\
\{gongxu.luo, boyang.sun, kun.zhang\}@mbzuai.ac.ae\\
}

\begin{document}

\maketitle

\begin{abstract}
Bulk gene expression profiling, which aggregates pooled RNA across cells within a biological sample, remains important in the single-cell era because it is typically less noisy, more sensitive, and more cost-effective than single-cell assays. Accordingly, a growing body of computational methods seeks to recover causal relations among genes from bulk expression data. However, aggregation is a lossy, non-invertible coarsening of the underlying cellular system, and it remains unclear whether and under what conditions causal relations are recoverable from aggregated bulk gene expression data. To answer this, we formalize recoverability under aggregation through two notions of consistency: functional-form consistency and conditional-independence consistency. We then derive necessary and sufficient conditions for recoverability, showing that these properties are preserved under linear aggregations (e.g., sum/mean) coupled with affine structural equations. To assess the practical plausibility of these conditions, analyses of four bulk and four single-cell gene expression datasets further reveal that the estimated pairwise regulatory functions among genes deviate from linearity in both data types, providing limited empirical support for the linearity assumptions required for recoverability. Together, these results caution against recovering causal relations from aggregated bulk expression data without strong additional assumptions.
\end{abstract}

\section{Introduction}
Understanding causal relations among genes from expression data is a central goal in computational biology~\cite{davidson2008properties}. Despite the growing availability of single-cell assays, bulk gene expression profiling remains prevalent in large-scale studies, yielding measurements aggregated over cell populations~\cite{wilks2021recount3,zhang2025perturbatlas}. For example, bulk RNA-seq quantifies gene expression from RNA pooled across many cells in a tissue sample or cultured population, effectively measuring the summed transcript abundance over the mixture, as shown in the blood example in Figure~\ref{fig:exp}(b).

Over the past decades, the continued prevalence of bulk gene expression data in biological research has motivated a growing body of computational approaches, including information theory models~\cite{margolin2006aracne,zhang2012inferring}, Bayesian networks~\cite{zou2005new,liu2016inference}, and neural networks~\cite{wu2024spred,zhu2024hybrid}, which aim to discover linear or nonlinear causal relations among genes from bulk expression profiles~\cite{delgado2019computational}. However, most methods typically treat aggregated bulk measurements as direct observations of a single homogeneous causal system. As a result, algorithmic output may correspond only to the causal model implicit in the aggregated distribution.\looseness=-1

An overlooked central challenge is that aggregation can distort the key identifying information for causal discovery, including functional form, statistical dependence, and conditional independence (CI), among the aggregated variables in ways that differ from those implied by the original causal model. For instance, in the original causal graph among three genes $X$, $Y$, $Z$ in Figure~\ref{fig:exp}(a), $X \indep Z|Y$. After aggregation over $m$ random samples, however, we observe only the aggregated counterparts $\bar{X}, \bar{Y}, \bar{Z}$, for which the CI relation is violated, specifically, $\bar{X} \dep \bar{Z}|\bar{Y}$, as shown in (c). This challenge raises a fundamental recoverability question: when do causal relations inferred from bulk expression remain consistent with the underlying causal structure, and under what conditions can this recoverability be guaranteed? Resolving this question is crucial for interpreting causal claims from aggregated bulk cohorts and for developing computational pipelines with principled correctness guarantees.\looseness=-1

\begin{figure*}
    \centering
    \includegraphics[width=0.9\textwidth]{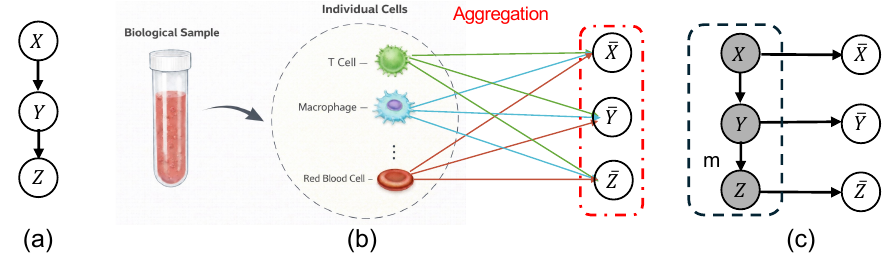}
    \caption{Illustration of aggregated bulk gene expression data. (a) The causal relationships among genes $X$, $Y$, and $Z$. (b) The aggregation process that generates bulk gene expression data. (c) The causal graph of hidden genes and observed aggregated counterparts $\bar{X}, \bar{Y}, \bar{Z}$ over $m$ random draws.\looseness=-1}
    \label{fig:exp}
     \vspace{-0.5cm}
\end{figure*}

To address this, we develop a framework that makes the recoverability of causal relations from aggregated bulk expression precise by formalizing the relation between the underlying causal model and the causal structure implied by aggregated measurements. We formulate two complementary consistency notions under aggregation: functional-form consistency, which asks whether the aggregated variables admit a structural causal model (SCM) of the same functional form as the underlying model in a way that preserves the intended causal relations, and conditional-independence consistency, which asks whether the conditional (in)dependence in aggregated measurements is consistent with the underlying causal model. Within this framework, we derive necessary and sufficient conditions for recoverability and show that they are preserved under linear aggregations (e.g., sum/mean) coupled with affine structural equations. Finally, we translate these theoretical requirements into statistical diagnostics of functional linearity, based on the inverse effective degrees of freedom (edf) from generalized additive models and the deviation of the estimated pairwise functions from their best affine approximations, and apply them to real bulk and single-cell gene expression datasets. The empirical results show that functional linearity between genes is generally weak in both settings, suggesting that the conditions required for recoverability are unlikely to hold in practice and that causal relations cannot be reliably recovered from bulk expression data without additional assumptions.\looseness=-1

\textbf{Contributions.}\label{contribution}
This paper addresses a fundamental but underexplored question in gene regulatory network inference: whether causal relations can be recovered from aggregated bulk gene expression data. We provide a theoretical characterization of this recoverability in terms of functional-form and conditional-independence consistency, and prove that it is guaranteed under linear aggregation and affine structural equations. Synthetic and real-data analyses further suggest that this regime is rarely met in practice, implying that causal relations are generally not reliably recoverable from bulk expression alone without additional assumptions.

\section{Related Work}
Recovering causal relations from bulk gene expression data is a typical causal discovery problem; causal discovery paradigms, including constraint-based, score-based, and function-based methods~\cite{glymour2019review}, are utilized to identify the causal relations among genes~\cite{delgado2019computational}. 

\noindent\textbf{Constraint-based causal discovery methods} rely on the principle that conditional independence implies the absence of a direct causal relation. These methods, represented by the PC algorithm, typically begin with a fully connected graph and iteratively remove edges when pairs of variables are found to be conditionally independent. Causal discovery methods have been widely applied to bulk gene expression data, with adaptations to address key challenges in this domain. For high-dimensional settings, parallelized implementations \cite{le2016fast} and stochastic complexity \cite{ji2024construction} have been developed to improve scalability. To better capture nonlinear dependencies, conditional mutual information–based independence tests have been incorporated \cite{ zhang2012inferring,liu2016inference}. When time-course measurements are available, temporal ordering constraints have been leveraged to aid causal edge orientation \cite{acerbi2014gene,ajmal2020inferring}. Moreover, biological information, such as flexible features \cite{saito2009co} and the Mendelian randomization principle \cite{badsha2021mrpc}, is incorporated to identify co-expression patterns and causal relations.

\noindent\textbf{Score-based causal discovery methods} formulate structure learning as an optimization problem, searching over candidate graphs and selecting the best one according to its fit to the observed data, typically with a penalty on model complexity. Gene expression analysis is formulated as Bayesian network structure learning, using Bayesian score search with the sparse candidate strategy to identify high-scoring regulatory networks~\cite{friedman2000using}. To mitigate the issue of model complexity in high-dimensional settings, mutual information, breakpoint detection  \cite{xing2017improved}, genetic node ordering \cite{wang2019high}, and skeleton constraint \cite{zhu2024hybrid} are applied to the representative score-based method GES to restrict the search space.

\noindent\textbf{Function-based causal discovery methods} assume a functional causal model and infer causal relations by exploiting asymmetries in the data-generating process or properties of the noise terms. Under the linear assumption, the structural equation model (SEM) with penalized likelihood is applied to extract a sparse causal structure~\cite{liu2008gene}. Moreover, genetic perturbation and sparsity-aware maximum likelihood are designed to identify causal direction and improve inference efficiency~\cite{cai2013inference}. Under the assumption of an additive noise model (ANM), the asymmetry~\cite{jiao2018bivariate} and genetic perturbation with SEM~\cite{li2021differential} are utilized to identify causal relations from bulk gene expression data.

\textbf{Causal abstraction and aggregation.} Another line of work studies the relation between causal models of micro-level variables and of their aggregates, referred to as macro-level variables. Aggregating causally heterogeneous variables, such as LDL and HDL into total cholesterol, has been shown to yield ambiguous manipulations~\cite{spirtes2004causal}, and macro-level variables have instead been learned from experimental data as coarsenings that preserve interventional effects~\cite{chalupka2016multi}. Exact transformations between SEMs have been introduced to characterize when models at different levels of description agree on interventional distributions, with averaging shown to yield such consistency for a class of linear SEMs~\cite{rubenstein2017causal}. For time series, causal structure has been shown to be identifiable from temporally aggregated data under linear non-Gaussian vector autoregressive models~\cite{gong2017causal}, and the consistency of causal relations inferred from such data with the underlying process has been examined~\cite{fan2024recoverability}.

Taken together, although causal discovery methods are widely applied to bulk gene expression data, existing methods treat bulk expression profiles as direct observations of the variables in a single underlying causal model, implicitly assuming that aggregation preserves causal semantics. However, bulk assays arise from a lossy, non-invertible aggregation over cell populations, which can violate this assumption and induce a mismatch between bulk-level and underlying causal relations. Moreover, prior analyses of causal aggregation, considering designed macro-level variables, interventional consistency, or temporal aggregation, do not resolve this issue. We therefore characterize, through analogous consistency notions, when causal relations are recoverable under cellular aggregation.


\section{The Aggregated Nature of Bulk Expression After Standard Normalization}
\label{sec:3}
In this section, we describe how bulk gene expression data are profiled and emphasize their inherently aggregated nature. Importantly, although normalization procedures can mitigate technology-specific biases introduced by different profiling platforms, the resulting bulk expression measurements remain aggregated signals over the cells contained in each biological sample as shown in Figure~\ref{fig:normalization}.

\begin{figure*}[h]
\centering
    \includegraphics[width=1\textwidth]{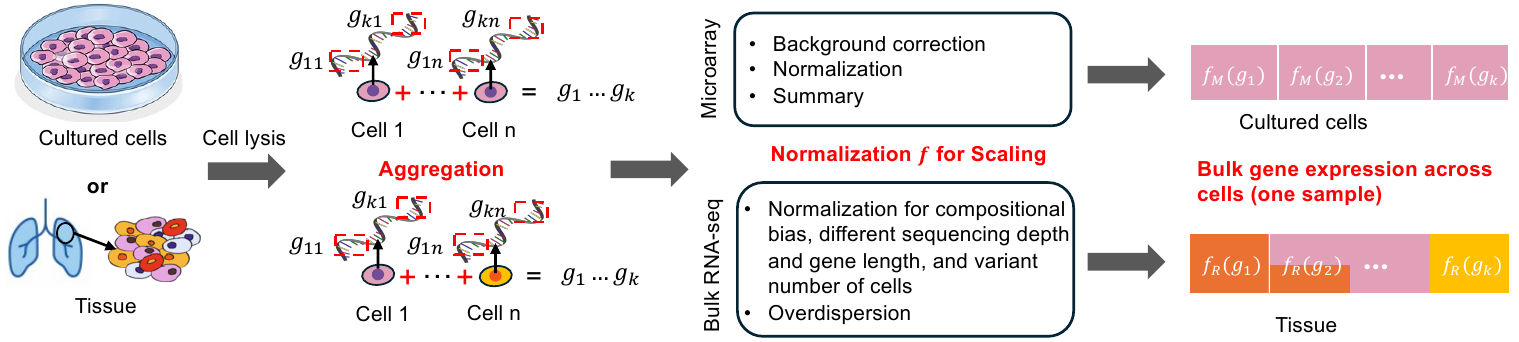}
    \caption{Illustration of the aggregated nature of bulk expression after standard normalization.}
    \label{fig:normalization}
\end{figure*}
\subsection{Microarray and normalization}

Microarray is a tissue-level gene expression profiling technology that measures aggregate gene expression across cells. Total RNA is extracted from a bulk biological sample, so transcripts from all constituent cells are pooled before measurement. The extracted RNA is converted into labeled cDNA or cRNA and hybridized to probes on a microarray chip. The raw expression data are then obtained from probe-specific fluorescence intensities~\cite{duggan1999expression,mccall2012affymetrix}, which reflect the summed abundance of the corresponding transcripts across cells, as illustrated in the first two steps of Figure~\ref{fig:normalization}.\looseness=-1

In microarray profiling, systematic technical variation, including global intensity bias, background noise, dye bias in two-color platforms, etc., can scale and bias the measurement of probe intensities~\cite{quackenbush2002microarray, do2006normalization}. The standard preprocessing procedure, consisting of background correction, normalization, and summarization, is designed to produce comparable gene expression estimates by transforming the aggregated (summed) measurements, preserving the aggregated nature.

Specifically, with a single-channel array platform as an example, given the specific hybridization signals $S_{ij}$ for probe $i$ in the $j$-th sample, where $i=0,\dots,m, j = 0,\dots,n$, the measured fluorescence intensities $D^{mi}_{ij}= S_{ij}+N_{ij}$, where $N_{ij}$ indicates the background noise. Background correction estimates the $S_{ij}$ from $D^{mi}_{ij}$. Then, to tackle the systematic technical variation, quantile normalization is designed on $S_{ij}$ as follows: for each sample $j$, rank the hybridization signals for all probes $S_{(1)j}\leq S_{(2)j} \leq \dots\leq S_{(m)j}$,  where the parenthesized subscript (r) indexes the rank after sorting. The averaged intensity at each rank is then computed across samples $\bar{S}_{(r)} = \frac{1}{n}\sum_{j=1}^n S_{(r)j}$, where rank $r=1,\dots,m$. Then, the normalized measurement is $\tilde{S}_{ij}= \bar{S}_{(r_{ij})}$, where $r_{ij}$ denotes the rank of $S_{ij}$ within sample $j$. Thus, all samples are normalized to a common distribution and are therefore comparable. Because each transcript is measured by multiple probes, probe-level intensities are summarized into a gene-level expression value. For gene \(k\) with probe set \(P_k\), the normalized log-intensity of probe \(i\in P_k\) in sample \(j\) can be modeled as
\(
\log_2(\tilde S_{ij}^k)=\theta_j^k+\psi_i+\sigma_{ij},
\)
where \(\theta_j^k\) denotes the underlying expression level of gene \(k\) in sample \(j\), \(\psi_i\) captures the probe-specific affinity effect, and \(\sigma_{ij}\) is residual noise. The gene-level expression estimate is then obtained as \(\tilde\theta_j^k\).

After background correction, denote quantile normalization by $f_M(\cdot)$ and the probe-set summarization, i.e., fitting the above model over $P_k$, by $s(\cdot)$. The normalized expression value of gene $k$ is as follows:
\begin{equation}\label{eq:microarray}
    \tilde{\theta}_{j}^k = s\Big(\big\{\log_2 f_M(S_{ij})\big\}_{i\in P_k}\Big) = s\Big(\big\{\log_2 f_M\big(\textstyle\sum_{c=1}^{C}\theta_{kc}^{ij}\big)\big\}_{i\in P_k}\Big),
\end{equation}
where $\theta_{kc}^{ij}$ indicates the expression value of gene $k$ of the $c$-th cell measured by probe $i$ in the $j$-th sampling, $C$ denotes the number of cells for one sampling. $f_M$ and $s$ do not change the aggregated nature.\looseness=-1

\subsection{Bulk RNA-seq and normalization}
Bulk RNA-seq is a sequencing-based transcriptomic assay that measures pooled mRNA from all cells in a biological sample. Total RNA extracted from the sample is prepared into a sequencing library through mRNA enrichment or rRNA depletion, fragmentation, reverse transcription, adapter ligation, and amplification. The library is then sequenced, and the resulting reads are mapped to a reference genome or transcriptome to quantify gene or transcript abundance~\cite{mortazavi2008mapping,wang2009rna}, reflecting aggregated expression signals.\looseness=-1

Aggregated bulk RNA-seq measurements are subject to sample-specific scaling differences arising from compositional bias~\cite{robinson2010scaling}, variation in sequencing depth~\cite{mortazavi2008mapping}, and differences in cell number~\cite{loven2012revisiting}. Standard preprocessing typically addresses these effects by estimating sample-specific normalization factors that place counts on a common scale, without altering the aggregated nature of bulk RNA-seq data as shown in Figure~\ref{fig:normalization}.

Taking DESeq2~\cite{love2014moderated} as an example, for gene $k$ in sample $j$, let $D^{se}_{kj}$ denote the observed read count. To account for sample-specific scaling effects, DESeq2 first estimates a size factor for each sample. Specifically, for each gene $k$, the geometric mean across all $m$ samples is defined as $G_k=\left(\prod_{j=1}^m D^{se}_{kj}\right)^{1/m}$. The size factor for sample $j$ is then estimated as
$s_j=\operatorname{median}_{k:G_k>0}\left(\frac{D^{se}_{kj}}{G_k}\right)$, and the corresponding normalized count is $\tilde{D}_{kj}=\frac{D^{se}_{kj}}{s_j}.$ To account for overdispersion, DESeq2 models the raw counts using a negative binomial distribution,
$D^{se}_{kj}\sim \mathrm{NB}(\mu_{kj},\gamma_k),$ where $\gamma_k$ is the gene-specific dispersion parameter and $\mu_{kj}=s_j q_{kj}.$ DESeq2 parameterizes the normalized mean $q_{kj}$ through a generalized linear model, $\log_2 q_{kj}=\sum_{r=1}^p z_{jr}\varsigma_{kr},$ where $z_{jr}$ is the $r$-th covariate of sample $j$, $\varsigma_{kr}$ is the corresponding regression coefficient for gene $k$, and $p$ is the number of covariates in the design matrix. Equivalently,
$\mu_{kj}=s_j\, 2^{\sum_{r=1}^p z_{jr}\varsigma_{kr}}.$
Under this model, the variance of $D^{se}_{kj}$ is
$\operatorname{Var}(D^{se}_{kj})=\mu_{kj}+\gamma_k\mu_{kj}^2.$\looseness=-1

For dispersion estimation, DESeq2 obtains a gene-wise estimate of $\gamma_k$, then fits a mean--dispersion trend across genes, and finally shrinks the gene-wise estimate toward the fitted trend. The regression coefficients $\varsigma_{kr}$ are then estimated within the negative binomial generalized linear model. These steps model the mean and variance of the counts for differential expression analysis rather than transform the count matrix, so the per-sample gene expression value is given by the normalized count, typically on the log scale. Denote this standard normalization procedure for bulk RNA-seq data by $f_R$, e.g., $f_R(D^{se}_{kj})=\log_2(D^{se}_{kj}/s_j+1)$. The normalized expression value of gene $k$ is as follows:
\begin{equation}\label{eq:rna-seq}
    \tilde{o}_j^k = f_R(D^{se}_{kj}) = f_R\Big(\sum_{c=1}^{C}o_{kc}^j\Big),
\end{equation}
where $o_{kc}^j$ indicates the read count of gene $k$ contributed by the $c$-th cell in the $j$-th sample, and $C$ denotes the number of cells in one sample. $f_R$ does not change the aggregated nature.\looseness=-1

For bulk gene expression data, including both microarray and bulk RNA-seq, existing normalization methods are primarily designed to correct sample-specific scaling effects and platform-specific measurement artifacts of the aggregated data. However, as discussed in Eq.~\ref{eq:microarray} and Eq.~\ref{eq:rna-seq}, these procedures do not alter the fundamental aggregation structure of bulk measurements.

\section{Recoverability of Causal Relation}
\label{sec:4}
Building on the aggregated nature of bulk gene expression data established above, we next study the recoverability of causal relations from bulk measurements. Specifically, we derive the necessary and sufficient conditions under which functional-form consistency and conditional-independence consistency are preserved under aggregation. Our analysis shows that such consistency can be guaranteed only for affine causal relations under linear aggregations.

\subsection{Problem formulation}

Let \( \mathbf{V} \in \mathbb{R}^d \) denote a latent random vector governed by a causal model with distribution \( P_\mathbf{V} \).
In our setting, the individual components of $\mathbf{V}$ are not directly observable; instead, we can only observe their aggregated counterparts. For any latent variable $X \in \mathbf{V}$, we define its observed aggregate realization \( \bar{x} \) as the result of a fixed component-wise mapping
$g : \mathbb{R}^m \to \mathbb{R}$
(e.g., sum, mean, median) applied to a collection of $m$ independent and identically distributed (i.i.d.) micro-samples.
Specifically, the $i$-th observed realization is given by
\begin{equation} \label{problem_dg}
\bar{x}_i = g\left( \{{x_j}\}_{j \in K_i} \right)
\end{equation}
where $K_i$ is a randomly selected index set of size $m$ such that $|K_i|=m$, and $K_i \cap K_{i'} = \emptyset$, for $i \not= i'$. We use $K$ to denote an index set of an arbitrary $m$ sampling from $P_\mathbf{V}$.

Given the observed aggregated data, the problem is to determine whether causal relations over \(V\) are recoverable after aggregation. We study this recoverability problem through consistency in directed functional relations and conditional independence under aggregation. 

\subsection{Functional-form consistency under aggregation}

Functional-form consistency concerns whether a causal relation remains in the same functional family after aggregation. Suppose that \(X\) causes \(Y\), and the causal relation is given by the general functional form
\(
Y=f(X,\varepsilon_Y),
\)
where \(\varepsilon_Y\) denotes the exogenous noise variable associated with \(Y\). Then, the functional-form consistency is defined as follows: 
\begin{definition}[Functional-form consistency under aggregation]
Suppose the structural equation is
\(
Y=f(X,\varepsilon_Y).
\)
Consider an arbitrary index set \(K=\{k_1,\dots,k_m\}\), define the aggregated results of $g$ over $K$ by
\(
\bar{x}=g\bigl((x_i)_{i\in K}\bigr), \bar{y}=g\bigl((y_i)_{i\in K}\bigr), \bar{\varepsilon}_y=g\bigl((\varepsilon_{y_i})_{i\in K}\bigr)
\).
We say that the structural equation \(f\) is \emph{functional-form consistent under aggregation \(g\)} if and only if (iff) there exists a function \(h:\mathbb{R}\times\mathbb{R}\to\mathbb{R}\) such that
\(
\bar Y=h(\bar X,\bar\varepsilon_Y)
\).
\end{definition}

Specifically, we consider two common functional causal model families: the additive noise model (ANM)
\(
Y=f(X)+\varepsilon_Y,
\)
and the post-nonlinear model (PNL)
\(
Y=f_2(f_1(X)+\varepsilon_Y),
\)
where \(\varepsilon_Y\indep X\). Before characterizing when these functional forms are preserved under aggregation, we distinguish between linear and nonlinear aggregations as follows:

\begin{definition}[Linear aggregation]
Let \(g:\mathbb{R}^m\to\mathbb{R}\) be a component-wise mapping. We say that \(g\) is a \emph{linear aggregation} iff, for any \(u,e\in\mathbb{R}^m\) and any \(\alpha,\beta\in\mathbb{R}\),
\(
g(\alpha u+\beta e)=\alpha g(u)+\beta g(e).
\)
\end{definition}

Equivalently, a linear aggregation can be written as a fixed weighted sum of the individual-level inputs:
\(
g((u_{i})_{i \in K})=\sum_{j=1}^m w_j u_{k_i},
\)
where \(w_1,\ldots,w_m\in\mathbb{R}\) do not depend on the input values. Sum aggregation and mean aggregation are special cases. Aggregations that cannot be represented in this form, such as median, maximum, minimum, or other order-based transformations, are referred to as \emph{nonlinear aggregations}. Then, the necessary and sufficient conditions for the recoverability of the functional causal models specified by functional-form consistency are as follows:\looseness=-1

\begin{restatable}[Recoverability of causal relations under aggregation for functional causal models]{theorem}
{THEOREM}\label{recoverability}
Assume that $X \indep \varepsilon_{Y}$, structural equations are measurable, and let \(((v_i)_{i\in K})\) be i.i.d. sampled from $P_V$. Let \(g:\mathbb{R}^m\to\mathbb{R}\) be a linear aggregation, that is,
\(
g(u+e)=g(u)+g(e), \forall u,e\in\mathbb{R}^m.
\)

\textbf{(i) ANM.}
Suppose the structural equation model is
\(
Y=f(X)+\varepsilon_Y
\). Define $F(x_K):=(f(x_i))_{i \in K}, x_K = \{x_i\}_{i \in K}$. Then ANM is functional-form consistent under aggregation if and only if there exists a function \(h:\mathbb{R}\to\mathbb{R}\) such that
\begin{equation}
    g(F(x_K)) = h(g(x_K)).
    \label{eq1}
\end{equation} 

\textbf{(ii) PNL.}
Suppose the structural equation model is
\(
Y=f_2(f_1(X)+\varepsilon_Y)
\), where \(f_2\) is invertible and sufficiently smooth. Define $F_1(x_K)=(f_1(x_i))_{i \in K}$ and $F_2(x_K)=(f_2(x_i))_{i \in K}, x_K = \{x_i\}_{i \in K}$. Then the PNL model is functional-form consistent under aggregation if and only if there exist functions $h_1$ and an invertible and sufficiently smooth $h_2$ such that
\begin{equation}
    g\!\left(F_2\!\bigl(F_1(x_K)+ \varepsilon_{y_K} \bigr)\right) = \\ h_2\!\left(h_1(g(x_K)) + g(\varepsilon_{y_K})\right), \varepsilon_{y_K} = (\varepsilon_{Y_i})_{i\in K}.
    \label{eq2}
\end{equation} 
Therefore, the causal relation is recoverable from the aggregated distribution, since the aggregated random variables satisfy
\(
\bar Y=h(\bar X)+\bar\varepsilon_Y
\) for ANM and \(
\bar Y=h_2(h_1(\bar X)+\bar\varepsilon_Y)
\) for PNL separately, thereby preserving the functional form of the corresponding structural equation model.
\end{restatable}

Theorem~\ref{recoverability} gives conditions for when the functional form of a causal model is preserved under aggregation. In the ANM case, since $g$ is linear, for any index set $K$,
\(
    \bar Y
    = g(y_K)
    = g(F(x_K)+\varepsilon_{y_K})
    = g(F(x_K)) + g(\varepsilon_{y_K}).
\)
Moreover, since $X_K \indep \varepsilon_{y_K}$, the aggregated noise term  $\bar\varepsilon_Y=g(\varepsilon_{y_K})$ remains independent of the aggregated cause $\bar X=g(x_K)$. Therefore, the aggregated variables still follow an ANM if and only if 
$g(F(x_K))$ can be written only as a function of $g(x_K)$, that is,
\(
    g(F(x_K)) = h(g(x_K)),
\)
which gives Eq.~\ref{eq1}. The PNL case follows the same idea: aggregation preserves the PNL form exactly when the aggregated variables can still be written as
\(
    \bar Y = h_2\!\left(h_1(\bar X)+\bar\varepsilon_Y\right),
\)
which leads to Eq.~\ref{eq2}. Hence, when these conditions hold, the functional form of the causal model is consistent after aggregation, and the causal relation can be recovered from the aggregated distribution. The detailed proof is provided in Appendix~\ref{proof:recoverability}.

With this theoretical foundation, a natural question is then: what kind of structural equations $f$ can satisfy these conditions? In the following, we discuss the conditions of structural equations that meet the proposed theorem for ANM and PNL separately. 

\begin{restatable}{corollary}{coroanm}
\label{corollary-ANM}
For ANM, suppose \(g\) is a non-degenerate linear aggregation. If functional-form consistency is preserved under aggregation, then \(f\) must be affine, i.e., \(f(X)=\alpha X+\beta\).
\end{restatable}

This corollary shows that functional-form consistency under linear aggregation is very restrictive. Since a linear aggregation can be written as \(g(x_K)=\sum_{j=1}^m w_j x_{k_j}\), Eq.~\ref{eq1} requires the weighted aggregate of \(f(x_{k_j})\) to depend only on the weighted aggregate of \(x_{k_j}\). For nonlinear \(f\), this is generally impossible, because two different configurations can have the same aggregate \(g(x_K)\) but different values of \(g(F(x_K))\). Affine functions are the only regular functions that avoid this problem: if \(f(x)=\alpha x+\beta\), then \(g(F(x_K))=\alpha g(x_K)+\beta\sum_{j=1}^m w_j\), which is a function of \(g(x_K)\) only. Hence, functional-form consistency under non-degenerate linear aggregation forces \(f\) to be affine. The detailed proof is provided in Appendix~\ref{proof:corollory 4.4}.

\begin{restatable}{corollary}{COROPNL}
\label{corollary-PNL}
For PNL, suppose \(g\) is a non-degenerate linear aggregation. If functional-form consistency is preserved under aggregation, then both \(f_1\) and \(f_2\) must be affine.
\end{restatable}

The PNL case inherits the same restriction as the ANM case, but the restriction applies in two steps. First, when the input variables are fixed, the remaining variation comes only from the noise terms, so Eq.~\ref{eq2} imposes an ANM-type aggregation constraint on the outer transformation \(f_2\). By the same reasoning as in the ANM corollary, this forces \(f_2\) to be affine. Once \(f_2\) is affine, the PNL condition reduces to the same type of aggregation constraint on \(f_1\). Applying the ANM corollary again implies that \(f_1\) must also be affine. Therefore, preserving the PNL form under non-degenerate linear aggregation requires both \(f_1\) and \(f_2\) to be affine. The detailed proof is provided in Appendix~\ref{proof:corollory 4.5}.


\subsection{Conditional-independence consistency}

In this section, we move beyond the two-variable setting and study whether conditional independence relations are preserved when passing from latent micro-variables to aggregated macro-variables. In particular, we examine the consistency between the conditional independence structure of latent variables and that of their observable aggregates.

Without loss of generality, we consider the three-variable case.
Let the observable variables $\bar{X}, \bar{Y}, \bar{Z}$ be formed as Equation~\eqref{problem_dg} by aggregating (e.g., summing) an unknown collection of samples from the latent variables $X, Y, Z$. The central question is: under what conditions do the observed variables preserve the same conditional independence relationships as the latent variables, so that causal discovery based on $\bar{X}, \bar{Y}, \bar{Z}$ remains meaningful and reliable?
We focus on the two fundamental types of conditional relationships among latent variables (see Figure~\ref{structure}): $X \indep Z | Y(\text{chain or fork)}, X \not \indep Z|Y (\text{collider})$.\looseness=-1

We first consider the collider structure (case (c)) shown in 
Figure~\ref{structure}, where $Y$ is a common effect \begin{wrapfigure}{r}{0.38\textwidth}
  \centering
\includegraphics[width=0.35\textwidth]{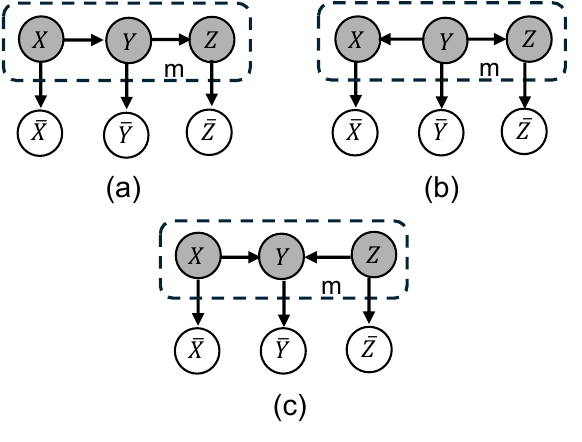}
  \caption{The three fundamental causal structures where we can only observe white nodes.}
  \label{structure}
   \vspace{-0.2cm}
\end{wrapfigure}of $X$ and $Z$. Under the faithfulness assumption and the Markov condition, conditioning on $Y$ (or any of its descendants) opens the path between $X$ and $Z$, rendering them conditionally dependent. While $\bar{Y}$ is the descendant of $Y$, the observed variables satisfy $\bar{X} \not \indep \bar{Z} | \bar{Y}$, which is consistent with the true underlying latent relationship $X \not \indep Z |Y$.

However, for the chain and fork cases, conditioning on the latent variable $Y$ blocks the path between $X$ and $Z$. In contrast, conditioning on its aggregated counterpart $\bar{Y}$ does not block the corresponding path between $\bar{X}$ and $\bar{Z}$. As a result, although the true latent relationship satisfies
$X \indep Z \mid Y,$ the observed variables may exhibit
$\bar{X} \not\indep \bar{Z} \mid \bar{Y}.$
This discrepancy indicates that conditional independence is not necessarily preserved under aggregation. Therefore, it is essential to characterize the conditions under which such equivalence holds. To this end, we propose the following theorem:

\begin{restatable}{theorem}{ciequ}
\label{ci equivalence}
    Let $X, Y, Z$ be micro-level latent variables, and let their macro-level aggregates be generated as in \eqref{problem_dg}. Assume that the micro-level conditional independence holds, i.e., $X \indep Z \mid Y$. Then the macro-level conditional independence holds if and only if, for every pair of values $\bar{x}$ and $\bar{z}$,
    \begin{equation}
        \operatorname{Cov}_{Y \mid \bar{Y}}
        \left(
        P(\bar{X}=\bar{x} \mid Y),
        P(\bar{Z}=\bar{z} \mid Y)
        \right) = 0.
    \end{equation}
\end{restatable}

The proof is provided in Appendix~\ref{app:ci eq}. Intuitively, the theorem identifies a “no residual co-variation” reflecting what happens when conditioning on $\bar{Y}$. Unlike conditioning on $Y$, fixing $\bar{Y}$ still leaves uncertainty about the latent variable $Y$, described by $Y \mid \bar{Y}$. This remaining variability can influence both $\bar{X}$ and $\bar{Z}$. The key issue is whether this influence is \emph{coordinated}. If different values of $Y$ consistent with the same $\bar{Y}$ tend to simultaneously increase or decrease the probabilities of $\bar{X}$ and $\bar{Z}$, then $\bar{X}$ and $\bar{Z}$ become dependent given $\bar{Y}$. 
The covariance term measures exactly this effect. 

If the covariance is nonzero, then there exist latent configurations of $Y$ that simultaneously push $\bar{X}$ and $\bar{Z}$ in the same (or opposite) direction, creating dependence at the aggregate level.
If the covariance is \textbf{zero}, then although $Y$ still varies given $\bar{Y}$, its influence on $\bar{X}$ and $\bar{Z}$ does not align: knowing that $Y$ makes $\bar{X}$ more likely tells us nothing about how it affects $\bar{Z}$, and vice versa.

An obvious sufficient condition for zero covariance is that at least one conditional probability, $P(\bar{X}=\bar{x} \mid Y)$ or $P(\bar{Z}=\bar{z} \mid Y)$, is constant with respect to $Y$. While this is generally difficult to satisfy, it can be met by constraining the functional relationships among the latent variables. Specifically, in linear systems, the macro-level variables inherit the conditional independence structure of the micro-level variables, as shown in the following corollary. Without loss of generality, we consider the fork case:

\begin{restatable}{corollary}{linearcici}
\label{linear_ci}
    Assume the latent variables $X, Y, Z$ satisfy a linear structural relationship:
    \begin{equation}
        X = \alpha Y + \epsilon_X, \quad Z = \beta Y + \epsilon_Z,
    \end{equation}
    where $\epsilon_X, \epsilon_Z$ are independent noise terms. Let $\bar{X}, \bar{Y}, \bar{Z}$ be the macro-level sum. Then, the micro-level conditional independence $X \indep Z \mid Y$ implies the macro-level conditional independence:
    \begin{equation}
        \bar{X} \indep \bar{Z} \mid \bar{Y}.
    \end{equation}
\end{restatable}
The proof is provided in Appendix~\ref{app:linear}. The corollary indicates that, in practical settings, when the true underlying relationships among variables are linear, causal discovery based on aggregated data remains meaningful and reliable, as the conditional independence structure is preserved. 
From another perspective, when the relationships are nonlinear, the condition in Theorem~\ref{ci equivalence} is generally not satisfied, and causal conclusions drawn from aggregated variables can be misleading. This phenomenon is further validated in our synthetic experiments, where nonlinear data-generating processes lead to violations of conditional independence at the aggregate level.

\section{From Synthetic Data to Bulk Gene Expression: Empirical Evaluation}
Having established the theoretical results, we now investigate their empirical implications. We begin with synthetic datasets, where the ground-truth causal structure is known, to verify the proposed necessary and sufficient conditions for causal recoverability under aggregation. We then analyze real bulk gene expression data to examine whether these conditions hold in practice and whether causal relations can be recovered from aggregated observations.
\subsection{Synthetic data: validation of recoverability under aggregation}

To validate the proposed necessary and sufficient conditions for causal recoverability under aggregation, we conduct a series of controlled simulation studies. We consider both the theoretically recoverable setting, namely, linear aggregation combined with affine causal relations, and deliberately misspecified settings involving nonlinear aggregation and/or nonlinear causal mechanisms. This design allows us to systematically examine when aggregation preserves the assumptions required for causal discovery and when it violates them.

For \emph{functional-form consistency}, we consider two classes of structural causal models. Under the ANM, data are generated according to
\(
Y = f(X) + \varepsilon_Y,
\)
where, in the nonlinear setting, \(f\) is selected from the function class
\(
\{\exp(x),\, x^p,\, \mathrm{ReLU}(wx+b),\, \sin(x)\},
\)
and, in the affine setting, \(f(x)=ax+b\), with coefficients \(a\) and \(b\) randomly sampled in each trial. Under the PNL model, data are generated according to
\(
Y = f_2\!\bigl(f_1(X)+\varepsilon_Y\bigr),
\)
where \(f_1\) is either selected from the same nonlinear function class or specified as an affine function, and \(f_2\) is an invertible function chosen from
\(
\{\exp(x),\, x^3,\, \sinh(x)\}.
\)
The noise term \(\varepsilon_Y\) is sampled from a Gaussian, skewed, or uniform distribution. For the aggregation step, we consider both a linear operation (mean of subsamples) and a nonlinear operation (median of subsamples).

For \emph{conditional-independence consistency}, we generate multivariate structural causal models of the form
\(
X = f\!\bigl(X_{\mathrm{pa}(X)}\bigr) + \varepsilon_X,
\)
and examine two representative causal motifs: the fork structure
($X\leftarrow Y \rightarrow Z$) and the collider structure (\(
X \to Y \leftarrow Z\)). All other experimental settings follow those used in the functional-form consistency experiments. We conduct 1000 experiments for each setting. The precision of recovering causal direction and conditional dependencies, reported in Table~\ref{tab1}, shows that, consistent with theoretical analysis, only linear causal relationships combined with linear aggregation preserve both the functional assumptions required by ANM- and PNL-based methods and the conditional-independence relations exploited by constraint-based approaches. 
\begin{table*}[t]
\centering
\caption{Precision of function-based and constraint-based causal discovery methods before (raw) and after aggregation (agg) over 1,000 experimental runs, all Unconditional Independent (UI) and Conditional Independent (CI) tests use a p-value threshold of 0.05.}
\label{tab1}
\scriptsize
\setlength{\tabcolsep}{4pt}
\resizebox{\textwidth}{!}{%
\begin{tabular}{llcccccccccccc}
\toprule
\multirow{2}{*}{Causation} & \multirow{2}{*}{Aggregation}
& \multicolumn{2}{c}{ANM}
& \multicolumn{2}{c}{PNL}
& \multicolumn{2}{c}{Fork-UI}
& \multicolumn{2}{c}{Fork-CI}
& \multicolumn{2}{c}{Collider-UI}
& \multicolumn{2}{c}{Collider-CI} \\
\cmidrule(lr){3-4} \cmidrule(lr){5-6} \cmidrule(lr){7-8} \cmidrule(lr){9-10} \cmidrule(lr){11-12} \cmidrule(lr){13-14}
& & raw & agg & raw & agg & raw & agg & raw & agg & raw & agg & raw & agg \\
\midrule
linear    & linear (mean)    & 1.000 & 0.830 & 1.000 & 0.912 & 0.999 & 0.994 & 0.943 & 0.932 & 0.953 & 0.955 & 1.000 & 0.999 \\
linear    & nonlinear (median) & 1.000 & 0.613 & 1.000 & 0.350 & 0.999 & 0.971 & 0.960 & 0.256 & 0.951 & 0.953 & 1.000 & 0.930 \\
nonlinear & linear (mean)   & 0.958 & 0.450 & 0.838 & 0.554 & 0.884 & 0.665 & 0.846 & 0.402 & 0.956 & 0.948 & 0.932 & 0.508 \\
nonlinear & nonlinear (median) & 0.967 & 0.591 & 0.840 & 0.460 & 0.850 & 0.414 & 0.865 & 0.617 & 0.937 & 0.961 & 0.953 & 0.432 \\
\bottomrule
\end{tabular}%
}
\end{table*}

\subsection{Real bulk gene expression data: empirical assessment of the recoverability Conditions}\label{sec:5.2}
After validating the theoretical results on synthetic data, we next examine whether real bulk gene expression data are compatible with the recoverable conditions (linear aggregation with affine regulatory function) identified in Section~\ref{sec:4}.  As discussed in Section~\ref{sec:3}, standard normalization does not alter the fact that bulk expression measurements arise from the sum (linear) aggregation across cells. Because the ground-truth gene regulatory network is unknown for real transcriptomic data, our aim here is not to establish recoverability directly, but to assess whether another key theoretical prerequisite, approximate linearity in regulation among genes, holds in practice.

\paragraph{Datasets.} To test whether the relationships among genes are approximately linear, we analyze four bulk gene expression datasets with relatively large sample sizes: two microarray datasets, GSE39582 (585 samples)~\cite{marisa2013gene} and GSE20142 (1240 samples)~\cite{dubois2010multiple}, and two bulk RNA-seq datasets, GSE57945 (260 samples)~\cite{haberman2014pediatric} and GSE245006 (518 samples)~\cite{stokes2024subtype}. The use of bulk data is motivated by the principle above: \textit{if the underlying regulatory relations are predominantly affine, traces of this linear structure should remain detectable in aggregated bulk observations}, because linear relations are preserved under linear mixing. Hence, bulk expression profiles can be used to provide indirect empirical evidence regarding whether the underlying gene dependencies admit a good linear approximation. As complementary evidence, we also analyze four single-cell gene expression datasets, Adamson~\cite{adamson2016multiplexed}, Norman~\cite{norman2019exploring}, Srivatsan~\cite{srivatsan2020massively}, and Tian~\cite{tian2021genome}. Since single-cell data are closer to the underlying cellular-level gene expression distributions, they offer a more direct assessment of whether the true relationships among genes are approximately linear. 

\begin{figure*}
\centering
\includegraphics[width=1\textwidth]{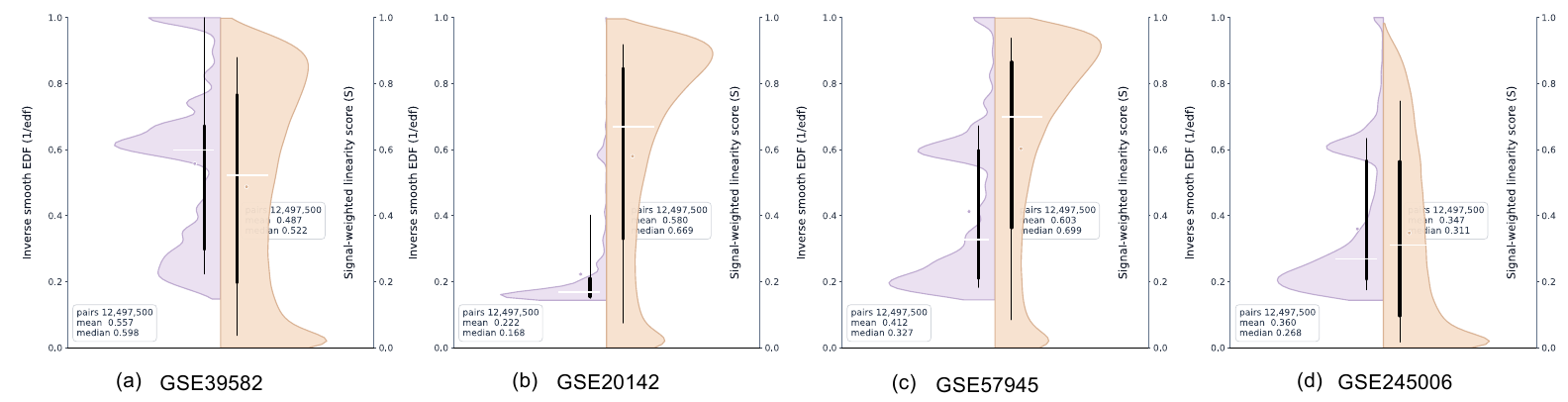}
\caption{Distribution of \textit{inverse smooth edf} and \textit{signal-weighted linearity score} across microarray (a) $\&$ (b) and bulk RNA-seq (c) $\&$ (d) datasets.}
\label{fig1:pearson-r2-distribution}
\end{figure*}

\paragraph{Preprocessing and linearity assessment.}

For the microarray datasets, we apply standard preprocessing to obtain normalized gene-level expression values, including background correction and normalization, together with probe summarization where applicable. For the bulk RNA-seq datasets, raw counts are converted to library-size-normalized log$_2$ counts per million (log$_2$-CPM). Our preprocessing is aligned with standardized pipelines provided by Gemma~\cite{lim2021curation}. To reduce noise and focus on informative variation, we first remove lowly expressed genes, retaining genes with expression values greater than 5 in at least 20 samples, and then exclude genes with low variability across samples. We finally retain the top 5,000 genes ranked by interquartile range (IQR) for downstream analysis.

We quantify pairwise functional linearity between gene pairs using two complementary diagnostics (details in Appendix~\ref{appendix:linearity-assessment}): the inverse effective degrees of freedom from generalized additive models, \(1/\mathrm{edf}\), which reflects the flexibility required by the fitted smooth function, and a signal-weighted linearity score, \(S\), which measures the degree of linearity. The distributions of these diagnostics for the four bulk datasets are shown in Fig.~\ref{fig1:pearson-r2-distribution}. Across datasets, most gene pairs have low \(1/\mathrm{edf}\) and \(S\), indicating that strong pairwise functional linearity is not widespread in bulk expression data. Furthermore, to support the conclusion from bulk gene expression data, we apply the same analysis to the Adamson, Norman, Srivatsan, and Tian single-cell datasets shown in Fig.~\ref{fig2:sc}. Because single-cell RNA-seq measures gene expression at cellular resolution, it provides a more direct assessment of pairwise functional linearity without the smoothing effects induced by bulk aggregation. 
\begin{figure*}
\centering
\includegraphics[width=1\textwidth]{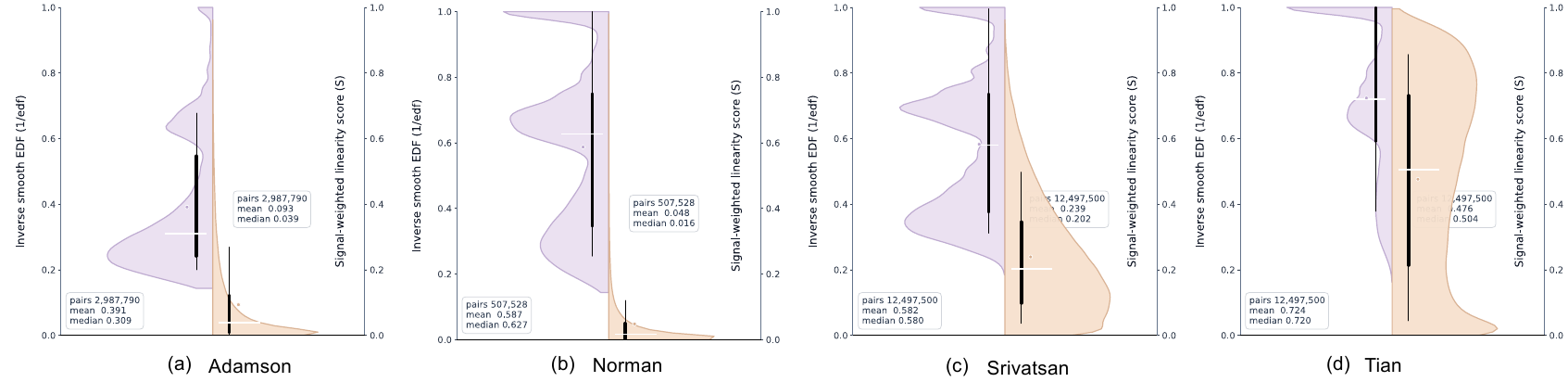}
\caption{Distribution of \textit{inverse smooth edf} and \textit{signal-weighted linearity score} across four canonical single-cell datasets.\looseness=-1}
\label{fig2:sc}
\end{figure*}

Experimental results show that these distributions are even more concentrated near zero compared to bulk gene expression data. One plausible explanation is bulk aggregation: each bulk profile is obtained by summing or averaging expression across many cells, and under suitable regularity conditions, the multivariate central limit theorem implies that the aggregated expression vector is better approximated by a multivariate Gaussian distribution. Such aggregation can smooth irregularities and attenuate higher-order nonlinear structures in estimated gene-to-gene functions, making them more amenable to affine approximation. Consequently, bulk data tend to exhibit higher apparent functional linearity, with GAM smooth terms having effective degrees of freedom closer to one and estimated pairwise functions deviating less from their best affine approximations.

Nevertheless, even with this aggregation-induced increase in apparent linearity, \(1/\mathrm{edf}\) and \(S\) remain low for most gene pairs in the bulk datasets. This suggests that the recoverability condition required by our theory is unlikely to hold broadly in real transcriptomic data, consistent with the synthetic experiments in Table~\ref{tab1}.

\section{Conclusion}
This work clarifies the fundamental limits of recovering causal relations from aggregated bulk gene expression data. We show that recovering underlying causal relations from aggregated measurements is possible only in a restricted linear regime, and that this regime appears rarely plausible in real transcriptomic data. These results establish a principled boundary for what can and cannot be inferred from bulk expression alone, and highlight the need for additional assumptions or auxiliary information for reliable causal recovery. Limitations are discussed in Appendix~\ref{limiation}.


\bibliography{reference}
\bibliographystyle{unsrt}

\newpage
\appendix

\section{Proof for Functional-form Consistency}\label{proof:F}
\subsection{Proof for Theorem~\ref{recoverability}} \label{proof:recoverability}

\THEOREM*
\begin{proof}[Proof]
\textbf{(i) ANM.} \textbf{(`$\Rightarrow$' direction)}
Suppose first that the aggregated variables are functional-form consistent with the ANM family, i.e., there exists a function \(h:\mathbb{R}\to\mathbb{R}\) such that \(\bar Y = h(\bar X)+\bar\varepsilon_Y.\) Since the underlying structural equation is \(y_i = f(x_i)+\varepsilon_{Y_i}, i\in K\), we may write \(y_K = F(x_K)+\varepsilon_{y_K}\).
As $X\indep\varepsilon_Y$, by the additivity of \(a\),
\begin{equation}
\bar Y
= g(y_K)
= a\!\bigl(F(x_K)+\varepsilon_{y_K}\bigr)
= a\!\bigl(F(x_K)\bigr)+g(\varepsilon_{y_K}).
\end{equation}
Then,
\(a\!\bigl(F(x_K)\bigr)+g(\varepsilon_{y_K})=h(\bar X)+\bar\varepsilon_Y
= h(g(x_K))+g(\varepsilon_{y_K})\). Hence, $g(F(x_K)) = h(g(x_K))$. Therefore, Eq.~\ref{eq1} is necessary condition of functional-from consistency.

\textbf{(`$\Leftarrow$' direction)} Conversely, assume that there exists a function \(h:\mathbb{R}\to\mathbb{R}\) such that
\(a\!\bigl(F(x_K)\bigr)=h(g(x_K))\), for all \(x_K\in\mathbb{R}^{|K|}\). Then
\(\bar Y
= g(y_K)
= a\!\bigl(F(x_K)+\varepsilon_{y_K}\bigr)
= a\!\bigl(F(x_K)\bigr)+g(\varepsilon_{y_K})
= h(g(x_K))+g(\varepsilon_{y_K})
= h(\bar X)+\bar\varepsilon_Y\).
Thus, the aggregated variables are consistent in functional form of ANM, the sufficient condition holds.

\medskip
\textbf{(ii) PNL.} \textbf{(`$\Rightarrow$' direction)}
Suppose first that the aggregated variables are functionally consistent with the PNL family, i.e., there exist functions \(h_1:\mathbb{R}\to\mathbb{R}\) and \(h_2:\mathbb{R}\to\mathbb{R}\), with \(h_2\) invertible and sufficiently smooth, such that \(\bar Y=h_2\!\bigl(h_1(\bar X)+\bar\varepsilon_Y\bigr)\). Substituting \(\bar X=g(x_K)\), \(\bar\varepsilon_Y=g(\varepsilon_{y_K})\), and 
\(\bar{Y} = g(a_{y_K})=a\!\left(F_2\!\bigl(F_1(x_K)+\varepsilon_{y_K}\bigr)\right)
=
h_2\!\left(h_1(g(x_K))+g(\varepsilon_{y_K})\right)\).
Thus, Eq.~\ref{eq2} is the necessary condition of functional-form consistency.\looseness=-1

\textbf{(`$\Leftarrow$' direction)} Conversely, assume that there exist functions \(h_1:\mathbb{R}\to\mathbb{R}\) and \(h_2:\mathbb{R}\to\mathbb{R}\), with \(h_2\) invertible and sufficiently smooth, such that
\(a\!\left(F_2\!\bigl(F_1(x_K)+\varepsilon_{y_K}\bigr)\right)
=
h_2\!\left(h_1(g(x_K))+g(\varepsilon_{y_K})\right)\)
for all \(x_K\in\mathbb{R}^{|K|}\) and \(\varepsilon_{y_K}\in\mathbb{R}^{|K|}\). Then
\(
\bar Y
=
h_2\!\left(h_1(\bar X)+\bar\varepsilon_Y\right),
\)
which is precisely a PNL representation for the aggregated variables. Hence PNL functional-form consistency holds.

As in part (i), \(X_K\perp \varepsilon_{y_K}\) and measurability of \(a\) imply
\(
\bar X \perp \bar\varepsilon_Y.
\)
Therefore, the aggregated PNL model satisfies the required independence condition. Under the standard identifiability conditions for PNL models, the causal relation is recoverable from the aggregated data.
\end{proof}

\subsection{Proof of Corollary~\ref{corollary-ANM}}\label{proof:corollory 4.4}

\coroanm*
\begin{proof}
Considering the linear aggregation \(g\) on any subsample set \(x_K\), 
\(K=\{k_1,k_2,\dots,k_m\}\), since \(g\) is linear, it can be written as a fixed weighted sum
\(g(x_K)=\sum_{j=1}^m w_j x_{k_j}\), where \(w_1,\ldots,w_m\in\mathbb{R}\) do not depend on the input values. Since \(g\) is non-degenerate, without loss of generality, assume \(w_1\neq 0\) and \(w_2\neq 0\).

If the functional form is consistent under aggregation, it requires
\(g(F(x_K))=h(g(x_K))\), i.e., Eq.~\ref{eq1}. Therefore,
\(\sum_{j=1}^m w_j f(x_{k_j})
    =
    h\!\left(\sum_{j=1}^m w_j x_{k_j}\right)\).
Set \(x_{k_3}=\cdots=x_{k_m}=0\). Then
\begin{equation}
    w_1 f(x_{k_1})+w_2 f(x_{k_2})
    +\sum_{j=3}^m w_j f(0)
    =
    h(w_1x_{k_1}+w_2x_{k_2}).
    \label{eq4}
\end{equation}

Let \(f_0=f(0)\), define \(\phi(x):=f(x)-f_0\), and define
\(H(t):=h(t)-\left(\sum_{j=1}^m w_j\right)f_0\). Then Eq.~\ref{eq4} becomes
\begin{equation}
    w_1\phi(x_{k_1})+w_2\phi(x_{k_2})
    =
    H(w_1x_{k_1}+w_2x_{k_2}).
    \label{eq5}
\end{equation}
Now let \(x_{k_2}=0\). We obtain
\(H(w_1x_{k_1})=w_1\phi(x_{k_1})\). Similarly, letting \(x_{k_1}=0\), we obtain
\(H(w_2x_{k_2})=w_2\phi(x_{k_2})\). Substituting these two identities back into Eq.~\ref{eq5} yields
  $  H(w_1x_{k_1}+w_2x_{k_2})
    =
    H(w_1x_{k_1})+H(w_2x_{k_2})$.
Since \(w_1\neq 0\) and \(w_2\neq 0\), for any \(s,t\in\mathbb{R}\), we can set
\(s=w_1x_{k_1}\) and \(t=w_2x_{k_2}\). Hence,
\begin{equation}
    H(s+t)=H(s)+H(t),
    \qquad \forall\, s,t\in\mathbb{R}.
\end{equation}
Thus \(H\) satisfies the Cauchy functional equation. Since \(f\) is measurable, \(\phi\) is measurable, and \(H(w_1x)=w_1\phi(x)\) implies that \(H\) is measurable. By the classical regularity result for measurable additive functions, there exists a constant \(c\in\mathbb{R}\) such that
\(H(t)=ct\), for all \(t\in\mathbb{R}\).

Using \(H(w_1x)=w_1\phi(x)\), we have
\(w_1\phi(x)=cw_1x\). Since \(w_1\neq 0\), it follows that
\(\phi(x)=cx\). Therefore,
\[
    f(x)=\phi(x)+f_0=cx+f(0),
\]
which shows that \(f\) is affine.
\end{proof}


\subsection{proof of Corollary~\ref{corollary-PNL}}\label{proof:corollory 4.5}
\COROPNL*
\begin{proof}
Considering the linear aggregation \(g\) on the subsample set \(x_K\), 
\(K=\{k_1,k_2,\dots,k_m\}\), since \(g\) is linear, it can be written as a fixed weighted sum
    $g(z_K)=\sum_{j=1}^m w_j z_{k_j},$
where \(w_1,\ldots,w_m\in\mathbb{R}\) do not depend on the input values. Since \(g\) is non-degenerate, at least two weights are nonzero. Denote \(W:=\sum_{j=1}^m w_j\).

If the functional form is consistent under aggregation, then by Theorem~\ref{recoverability}, there exist functions \(h_1\) and invertible \(h_2\) such that, for \(x_K\) and realizations 
\(\varepsilon_K=(\varepsilon_{k_1},\dots,\varepsilon_{k_m})\), following Eq.~\ref{eq2},
\begin{equation}
    \sum_{j=1}^m w_j f_2\!\bigl(f_1(x_{k_j})+\varepsilon_{k_j}\bigr)
    =
    h_2\!\left(
    h_1\!\left(\sum_{j=1}^m w_j x_{k_j}\right)
    + \sum_{j=1}^m w_j \varepsilon_{k_j}
    \right).
    \label{eq:pnl_weighted}
\end{equation}

We first show that \(f_2\) is affine. Fix 
\(x_{k_1}=\cdots=x_{k_m}=x^\ast\), and define \(c:=f_1(x^\ast)\). Then Eq.~\ref{eq:pnl_weighted} reduces to
\begin{equation}
    \sum_{j=1}^m w_j f_2(c+\varepsilon_{k_j})
    =
    \tilde h\!\left(\sum_{j=1}^m w_j\varepsilon_{k_j}\right),
    \label{eq:pnl_f2}
\end{equation}
where 
\(\tilde h(t):=h_2\!\bigl(h_1(Wx^\ast)+t\bigr)\).
Define \(q(E):=f_2(c+E)\). Then Eq.~\ref{eq:pnl_f2} becomes
    $\sum_{j=1}^m w_j q(\varepsilon_{k_j})
    =
    \tilde h\!\left(\sum_{j=1}^m w_j\varepsilon_{k_j}\right)$. This has the same form as the condition in Corollary 1 under the same non-degenerate linear aggregation \(g\). Therefore, \(q\) must be affine. Hence, there exist constants \(\gamma,\eta\in\mathbb{R}\) such that
\(q(E)=\gamma E+\eta\). Therefore,
\(f_2(Z)=\gamma Z+\delta\), where \(\delta:=\eta-\gamma c\). Thus \(f_2\) is affine. Since \(f_2\) is invertible in the PNL model, we have \(\gamma\neq 0\).

Next, substituting \(f_2(Z)=\gamma Z+\delta\) into Eq.~\ref{eq:pnl_weighted}, we obtain
\begin{equation}
    \gamma \sum_{j=1}^m w_j f_1(x_{k_j})
    + \gamma \sum_{j=1}^m w_j\varepsilon_{k_j}
    + \delta W
    =
    h_2\!\left(
    h_1\!\left(\sum_{j=1}^m w_j x_{k_j}\right)
    + \sum_{j=1}^m w_j \varepsilon_{k_j}
    \right).
    \label{eq:pnl3}
\end{equation}

Now fix 
\(\varepsilon_{k_1}=\cdots=\varepsilon_{k_m}=\varepsilon^\ast\) 
for an arbitrary constant \(\varepsilon^\ast\in\mathbb{R}\). Then Eq.~\ref{eq:pnl3} reduces to
\begin{equation}
    \sum_{j=1}^m w_j f_1(x_{k_j})
    =
    \bar h\!\left(\sum_{j=1}^m w_j x_{k_j}\right),
\end{equation}
where
\(
\bar h(t):=
\frac{1}{\gamma}
\left[
h_2\!\bigl(h_1(t)+W\varepsilon^\ast\bigr)
-\gamma W\varepsilon^\ast-\delta W
\right]
\).
Again, this has the same form as the condition in Corollary 1 under the non-degenerate linear aggregation \(g\). Therefore, \(f_1\) must also be affine. Hence, both \(f_1\) and \(f_2\) are affine.
\end{proof}

\section{Proof for Conditional Independence Equivalence} \label{proof:CI}
\subsection{Proof for Theorem~\ref{ci equivalence}} \label{app:ci eq}

\ciequ*


By the Law of Total Probability, we can express the joint conditional distribution of the macro-variables by marginalizing out the micro-states of $\mathbf{Y}$, constrained by the observed macro-state $\bar{Y}$:
\[
P(\bar{X}, \bar{Z} \mid \bar{Y}) 
= \int P(\bar{X}, \bar{Z} \mid \mathbf{Y}, \bar{Y}) 
P(\mathbf{Y} \mid \bar{Y}) \, d\mathbf{Y}.
\]

Because the exact micro-state $\mathbf{Y}$ deterministically defines the $\bar{Y}$, conditioning on both is redundant. We can drop $\bar{Y}$ from the first term:
\[
P(\bar{X}, \bar{Z} \mid \bar{Y}) 
= \int P(\bar{X}, \bar{Z} \mid \mathbf{Y}) 
P(\mathbf{Y} \mid \bar{Y}) \, d\mathbf{Y}.
\]

Since $\bar{X}$ is purely a function of $\mathbf{X}$ and $\bar{Z}$ is purely a function of $\mathbf{Z}$, the micro-level independence 
\[
\mathbf{X} \perp\!\!\!\perp \mathbf{Z} \mid \mathbf{Y}
\]
implies
\[
P(\bar{X}, \bar{Z} \mid \mathbf{Y}) 
= P(\bar{X} \mid \mathbf{Y}) P(\bar{Z} \mid \mathbf{Y}).
\]

Substituting this back into the integral gives:
\[
P(\bar{X}, \bar{Z} \mid \bar{Y}) 
= \mathbb{E}_{\mathbf{Y} \mid \bar{Y}} 
\left[ P(\bar{X} \mid \mathbf{Y}) P(\bar{Z} \mid \mathbf{Y}) \right].
\]

For macro-level conditional independence to hold, this must factorize:
\[
P(\bar{X}, \bar{Z} \mid \bar{Y}) 
= P(\bar{X} \mid \bar{Y}) P(\bar{Z} \mid \bar{Y}).
\]

Thus,
\[
\mathbb{E}_{\mathbf{Y} \mid \bar{Y}} 
\left[ P(\bar{X} \mid \mathbf{Y}) P(\bar{Z} \mid \mathbf{Y}) \right]
=
\mathbb{E}_{\mathbf{Y} \mid \bar{Y}} 
\left[ P(\bar{X} \mid \mathbf{Y}) \right]
\mathbb{E}_{\mathbf{Y} \mid \bar{Y}} 
\left[ P(\bar{Z} \mid \mathbf{Y}) \right].
\]

This holds if and only if the covariance is zero:
\[
\operatorname{Cov}_{\mathbf{Y} \mid \bar{Y}} 
\left( 
P(\bar{X} \mid \mathbf{Y}), 
P(\bar{Z} \mid \mathbf{Y}) 
\right) = 0.
\]

\subsection{Proof of Corollary~\ref{linear_ci}} \label{app:linear}

\linearcici*

Assume the micro-level latent random variables $X, Y, Z$ follow the linear structural equations:
\begin{equation*}
    X = \alpha Y + \epsilon_X, \quad Z = \beta Y + \epsilon_Z
\end{equation*}

For $K$ independent realizations, let the specific samples be denoted by lowercase letters $x_j, y_j, z_j$ for $j \in \{1, \dots, K\}$. The structural equations for each sample are:
\begin{equation*}
    x_j = \alpha y_j + \epsilon_{x,j}, \quad z_j = \beta y_j + \epsilon_{z,j}
\end{equation*}

By definition, the realizations of the macro-variables are formed by summing these $m$ samples:
\begin{align*}
    \bar{x} &= \sum_{j=1}^K x_j = \sum_{j=1}^K (\alpha y_j + \epsilon_{x,j}) = \alpha \sum_{j=1}^K y_j + \sum_{j=1}^K \epsilon_{x,j} \\
    \bar{z} &= \sum_{j=1}^K z_j = \sum_{j=1}^K (\beta y_j + \epsilon_{z,j}) = \beta \sum_{j=1}^K y_j + \sum_{j=1}^K \epsilon_{z,j}
\end{align*}

Defining the aggregated noise realizations as $\bar{\epsilon}_x = \sum_{j=1}^K \epsilon_{x,j}$ and $\bar{\epsilon}_z = \sum_{j=1}^K \epsilon_{z,j}$, we obtain the macro-level relationships for the realizations:
\begin{equation*}
    \bar{x} = \alpha \bar{y} + \bar{\epsilon}_x, \quad \bar{z} = \beta \bar{y} + \bar{\epsilon}_z
\end{equation*}

Generalizing this back to the random variables yields the macro-level structural equations:
\begin{equation*}
    \bar{X} = \alpha \bar{Y} + \bar{\epsilon}_X, \quad \bar{Z} = \beta \bar{Y} + \bar{\epsilon}_Z
\end{equation*}

At the micro-level, the conditional independence assumption $X \indep Z \mid Y$ dictates that the exogenous noise variables $\epsilon_X$ and $\epsilon_Z$ are mutually independent, and both are independent of the latent cause $Y$. 

Because the $K$ samples are independently drawn, the aggregated noise variable $\bar{\epsilon}_X$ (the sum of independent $\epsilon_{x,j}$) is independent of $\bar{\epsilon}_Z$ (the sum of independent $\epsilon_{z,j}$). Furthermore, both $\bar{\epsilon}_X$ and $\bar{\epsilon}_Z$ remain completely independent of the aggregated latent cause $\bar{Y}$.

When conditioning on the macro-variable $\bar{Y} = \bar{y}$, the terms $\alpha \bar{y}$ and $\beta \bar{y}$ act as deterministic constants. Consequently, any remaining variation in $\bar{X}$ is driven strictly by $\bar{\epsilon}_X$, and any remaining variation in $\bar{Z}$ is driven strictly by $\bar{\epsilon}_Z$. 

Given that $\bar{\epsilon}_X \indep \bar{\epsilon}_Z$, it follows immediately that:
\begin{equation*}
    \bar{X} \indep \bar{Z} \mid \bar{Y}
\end{equation*}

\section{Assessment of pairwise functional linearity}
\label{appendix:linearity-assessment}
We assess the linearity of pairwise gene-gene relationships using two complementary measures. Let \(\mathcal{A}\in\mathbb{R}^{n\times p}\) denote the normalized expression matrix, where \(n\) is the number of samples or cells and \(p\) is the number of genes. For each gene pair \((X,Y)\), we treat \(X\) as the predictor and \(Y\) as the response. Let \(x_i=\mathcal{A}_{ij}\) and \(y_i=\mathcal{A}_{il}\) denote their expression values in sample or cell \(i\).

First, we fit a univariate generalized additive model (GAM)
\begin{equation}
    y_i = \mu + s(x_i) + \eta_i,
\end{equation}
where \(s(\cdot)\) is a smooth function estimated from the data. The effective degrees of freedom, denoted by \(\widehat{\mathrm{edf}}_{j\to l}\), measure the flexibility used by the fitted smooth function. An edf close to \(1\) indicates that the fitted function is approximately linear, whereas a larger edf indicates stronger nonlinearity. We therefore define the inverse-edf linearity score as
\begin{equation}
    L^{\mathrm{edf}}_{j\to l}
    =
    \frac{1}{\max\{\widehat{\mathrm{edf}}_{j\to l},1\}}.
\end{equation}
A larger value of \(L^{\mathrm{edf}}_{j\to l}\) indicates a stronger linearity.

Second, we directly compare the estimated pairwise function with its best affine approximation. Let
\(
\hat f_{j\to l}(x)=\hat\mu+\hat s(x)
\)
be the fitted function from the GAM. We evaluate this function at the observed predictor values and obtain
\(
\hat f_i=\hat f_{j\to l}(x_i)
\).
The best affine approximation is defined as
\begin{equation}
    (\hat\alpha,\hat\beta)
    =
    \arg\min_{\alpha,\beta}
    \sum_{i=1}^n
    \left[
        \hat f_i - (\alpha x_i+\beta)
    \right]^2 .
\end{equation}
We then compute the normalized affine-deviation score
\begin{equation}
    D^{\mathrm{aff}}_{j\to l}
    =
    \left(
    \frac{
    \sum_{i=1}^n
    \left[
        \hat f_i - (\hat\alpha x_i+\hat\beta)
    \right]^2
    }{
    \sum_{i=1}^n
    \left[
        \hat f_i - \bar f
    \right]^2
    }
    \right)^{1/2},
    \qquad
    \bar f=\frac{1}{n}\sum_{i=1}^n \hat f_i .
\end{equation}
A smaller value of \(D^{\mathrm{aff}}_{j\to l}\) indicates that the estimated function can be well approximated by an affine mapping. When the fitted function is exactly affine over the observed range of \(X_j\), \(D^{\mathrm{aff}}_{j\to l}=0\).

The quantity \(1-(D^{\mathrm{aff}}_{j\to l})^2\) is the fraction of the fitted variation captured by the affine approximation. It is therefore defined relative to the fitted signal and becomes unstable when the dependence between \(X\) and \(Y\) is weak. We thus weight it by the strength of the fitted signal and define the signal-weighted linearity score
\begin{equation}
    S_{j\to l}
    =
    \frac{(1+\epsilon)\left[1-(D^{\mathrm{aff}}_{j\to l})^2\right]R^2_{j\to l}}{R^2_{j\to l}+\epsilon},
    \qquad
    R^2_{j\to l}=\frac{\sum_{i=1}^n(\hat f_i-\bar f)^2}{\sum_{i=1}^n(y_i-\bar y)^2},
\end{equation}
where \(R^2_{j\to l}\) is the fraction of the variance of \(Y\) explained by the GAM fit and \(\epsilon>0\) is a small stabilizing constant (\(\epsilon=0.01\) in our experiments). \(S_{j\to l}\in[0,1]\), and a larger value indicates stronger linearity. When the fitted signal is strong (\(R^2_{j\to l}\gg\epsilon\)), \(S_{j\to l}\approx 1-(D^{\mathrm{aff}}_{j\to l})^2\); when it is weak (\(R^2_{j\to l}\ll\epsilon\)), \(S_{j\to l}\to 0\), so that gene pairs without detectable dependence are not counted as linear.

These two measures capture different aspects of linearity. The inverse-edf score reflects how much flexibility the GAM needs to fit the pairwise relationship, while the signal-weighted linearity score measures how much of the fitted relationship is captured by its closest affine approximation, discounting gene pairs with weak dependence. For each unordered gene pair, both scores are averaged over the two directions \(j\to l\) and \(l\to j\). Together, they provide a robustness check for whether the estimated pairwise regulatory functions are close to affine.



\section{Limitations and Discussion}\label{limiation}
\subsection{Discussion}
Bulk gene expression data provide an important and cost-effective way to study gene regulation, but they are obtained by aggregating RNA across cells. This aggregation removes cell-level variation and may alter the functional and statistical relationships among genes. In this work, we formalize when causal relations can be recovered from such aggregated data using functional-form consistency and conditional-independence consistency.

Our theoretical results show that recoverability under aggregation requires strong compatibility between the aggregation operator and the underlying causal mechanisms. In particular, under common linear aggregations such as sum or mean, the functional form is preserved only when the structural functions are affine. This indicates that causal relations inferred from bulk-level distributions should not be directly interpreted as cell-level causal relations unless these assumptions are justified.

The empirical analyses further support this caution. Across the examined bulk and single-cell datasets, estimated pairwise regulatory functions often deviate from linearity, suggesting limited empirical support for the affine assumptions required by the theory. Therefore, while bulk data remain useful for studying aggregated regulatory patterns, causal discovery from bulk expression data requires careful interpretation and strong additional assumptions.

\subsection{Limitations}
This work has two main limitations. First, although our theory characterizes when causal relations are recoverable under aggregation, it does not directly solve the practical problem of recovering causal relations from bulk gene expression data. In practice, reliable recovery may require additional information, such as perturbation data, time-series measurements, prior regulatory knowledge, or single-cell reference data. How to incorporate such information into causal discovery from bulk data remains an open implementation-level problem.

Second, the theoretical results in this work mainly focus on linear aggregation. While we show that recoverability under linear aggregation requires strong restrictions on the structural equations, the corresponding conditions for nonlinear aggregation remain unclear. In particular, for nonlinear aggregation operators, what forms of structural equations can preserve recoverability needs further theoretical investigation.


\end{document}